\documentclass[11pt]{article}

\usepackage[letterpaper, left=1in, right=1in, top=1in, bottom=1in]{geometry}

\usepackage{parskip}

\usepackage[dvipsnames]{xcolor}
\usepackage[colorlinks=true, linkcolor=blue, citecolor=blue,urlcolor=black,breaklinks=true]{hyperref}
\usepackage{breakcites}
\usepackage{microtype}

\usepackage{algorithm}

 \usepackage{natbib}
\bibpunct{(}{)}{;}{a}{,}{,}

\usepackage{amsthm}
\usepackage{mathtools}
\usepackage{amsmath}
\usepackage{bbm}
\usepackage{amsfonts}
\usepackage{amssymb}

\usepackage{MnSymbol} 

\usepackage{xpatch}

\theoremstyle{definition}  

\newtheorem{lemma}{Lemma}

\newtheorem{proposition}{Proposition}

\theoremstyle{plain}

\newtheorem{theorem}{Theorem}

\xpatchcmd{\proof}{\itshape}{\normalfont\proofnameformat}{}{}
\newcommand{\proofnameformat}{\bfseries}

\usepackage{prettyref}
\newcommand{\pref}[1]{\prettyref{#1}}

\newcommand{\savehyperref}[2]{\texorpdfstring{\hyperref[#1]{#2}}{#2}}
\newrefformat{eq}{\savehyperref{#1}{\textup{(\ref*{#1})}}}
\newrefformat{eqn}{\savehyperref{#1}{Equation~\ref*{#1}}}
\newrefformat{lem}{\savehyperref{#1}{Lemma~\ref*{#1}}}
\newrefformat{def}{\savehyperref{#1}{Definition~\ref*{#1}}}
\newrefformat{line}{\savehyperref{#1}{line~\ref*{#1}}}
\newrefformat{thm}{\savehyperref{#1}{Theorem~\ref*{#1}}}
\newrefformat{corr}{\savehyperref{#1}{Corollary~\ref*{#1}}}
\newrefformat{sec}{\savehyperref{#1}{Section~\ref*{#1}}}
\newrefformat{app}{\savehyperref{#1}{Appendix~\ref*{#1}}}
\newrefformat{ass}{\savehyperref{#1}{Assumption~\ref*{#1}}}
\newrefformat{ex}{\savehyperref{#1}{Example~\ref*{#1}}}
\newrefformat{fig}{\savehyperref{#1}{Figure~\ref*{#1}}}
\newrefformat{alg}{\savehyperref{#1}{Algorithm~\ref*{#1}}}
\newrefformat{rem}{\savehyperref{#1}{Remark~\ref*{#1}}}
\newrefformat{conj}{\savehyperref{#1}{Conjecture~\ref*{#1}}}
\newrefformat{prop}{\savehyperref{#1}{Proposition~\ref*{#1}}}
\newrefformat{proto}{\savehyperref{#1}{Protocol~\ref*{#1}}}
\newrefformat{prob}{\savehyperref{#1}{Problem~\ref*{#1}}}
\newrefformat{claim}{\savehyperref{#1}{Claim~\ref*{#1}}}

\DeclarePairedDelimiter{\abs}{\lvert}{\rvert} %

\DeclarePairedDelimiter{\crl}{\{}{\}}
\DeclarePairedDelimiter{\prn}{(}{)}
\DeclarePairedDelimiter{\nrm}{\|}{\|}
\DeclarePairedDelimiter{\tri}{\langle}{\rangle}

\DeclareMathOperator{\En}{\mathbb{E}}

\def\ddefloop#1{\ifx\ddefloop#1\else\ddef{#1}\expandafter\ddefloop\fi}
\def\ddef#1{\expandafter\def\csname bb#1\endcsname{\ensuremath{\mathbb{#1}}}}
\ddefloop ABCDEFGHIJKLMNOPQRSTUVWXYZ\ddefloop
\def\ddefloop#1{\ifx\ddefloop#1\else\ddef{#1}\expandafter\ddefloop\fi}
\def\ddef#1{\expandafter\def\csname b#1\endcsname{\ensuremath{\mathbf{#1}}}}
\ddefloop ABCDEFGHIJKLMNOPQRSTUVWXYZ\ddefloop
\def\ddef#1{\expandafter\def\csname c#1\endcsname{\ensuremath{\mathcal{#1}}}}
\ddefloop ABCDEFGHIJKLMNOPQRSTUVWXYZ\ddefloop
\def\ddef#1{\expandafter\def\csname h#1\endcsname{\ensuremath{\widehat{#1}}}}
\ddefloop ABCDEFGHIJKLMNOPQRSTUVWXYZ\ddefloop
\def\ddef#1{\expandafter\def\csname hc#1\endcsname{\ensuremath{\widehat{\mathcal{#1}}}}}
\ddefloop ABCDEFGHIJKLMNOPQRSTUVWXYZ\ddefloop
\def\ddef#1{\expandafter\def\csname t#1\endcsname{\ensuremath{\widetilde{#1}}}}
\ddefloop ABCDEFGHIJKLMNOPQRSTUVWXYZ\ddefloop
\def\ddef#1{\expandafter\def\csname tc#1\endcsname{\ensuremath{\widetilde{\mathcal{#1}}}}}
\ddefloop ABCDEFGHIJKLMNOPQRSTUVWXYZ\ddefloop

\newcommand{\ls}{\ell}
\newcommand{\pmo}{\crl*{\pm{}1}}
\newcommand{\eps}{\epsilon}
\newcommand{\veps}{\varepsilon}

\newcommand{\ldef}{\vcentcolon=}

\newcommand{\bigO}{\cO}
\newcommand{\bigOt}{\tilde{\cO}}

\newcommand{\xl}[1][n]{x_{1:#1}}

\newcommand{\Rad}{\mathfrak{R}}
\newcommand{\RadW}{\mathfrak{R}_n}

\newcommand{\midsem}{\,;}

\newcommand{\cFi}{\cF|_i}

\newcommand{\fat}{\mathrm{fat}}

\newcommand{\approxleq}{\lesssim}

\begin{document}

\title{$\ell_{\infty}$ Vector Contraction for Rademacher Complexity}

\author{Dylan J. Foster \\ {\small dylanf@mit.edu}\\
\and	
      Alexander Rakhlin \\{\small rakhlin@mit.edu} \\
}
\date{}

\maketitle

\begin{abstract}
We show that the Rademacher
complexity of any $\bbR^{K}$-valued function class composed with an
$\ls_{\infty}$-Lipschitz function is bounded by the maximum Rademacher
complexity of the restriction of the function class along each
coordinate, times a factor of $\bigOt(\sqrt{K})$.
\end{abstract}

\section{Introduction}
\label{sec:intro}

Rademacher complexity plays a fundamental role in
learning theory, where it tightly bounds the supremum of the empirical
process \citep{koltchinskii2000rademacher,bartlett2003rademacher} and is used to prove generalization guarantees for empirical risk minimization and other learning rules. Let $\cF$ be a class of functions $f:\cX\to\bbR$, and define the
empirical Rademacher complexity of $\cF$ for a sequence $\xl[n]=(x_1,\ldots,x_n)$ via
\begin{equation}
  \label{eq:1}
  \Rad(\cF\midsem\xl[n]) = \En_{\eps}\sup_{f\in\cF}\sum_{t=1}^{n}\eps_tf(x_t),
\end{equation}
where $\eps=(\eps_1,\ldots,\eps_n)$ is a sequence of i.i.d. Rademacher
random variables. A standard result used to bound Rademacher complexity for complex function
classes is the Lipschitz contraction inequality
\citep{ledoux1991probability}, which states (in its most basic form) that for any fixed sequence of $L$-Lipschitz mappings $\phi_1,\ldots,\phi_n$, 
we have
\begin{equation}
  \label{eq:lipschitz_contraction}
\En_{\eps}\sup_{f\in\cF}\sum_{t=1}^{n}\eps_t\phi_t\prn*{f(x_t)}
\leq{} L\cdot \En_{\eps}\sup_{f\in\cF}\sum_{t=1}^{n}\eps_tf(x_t).
\end{equation}
This result holds for the setting where the function class $\cF$ is real-valued, but there are many applications in which it is natural to work with
classes of vector-valued functions. \cite{maurer2016vector} recently proved a
vector-valued generalization of the contraction
inequality: When $\cF\subseteq\crl*{f:\cX\to\bbR^{K}}$ and $\phi_1,\ldots,\phi_n$ are
$L$-Lipschitz with respect to the $\ls_2$ norm, it holds that
\begin{equation}
\En_{\eps}\sup_{f\in\cF}\sum_{t=1}^{n}\eps_t\phi_t\prn*{f(x_t)}
\leq{}\sqrt{2} L\cdot
\En_{\eps}\sup_{f\in\cF}\sum_{t=1}^{n}\sum_{i=1}^{K}\eps_{t,i}f_i(x_t).\label{eq:maurer}
\end{equation}
This result has found numerous applications, including $K$-means clustering
\citep{fefferman2016testing}, robust learning
\citep{attias2019improved}, structured prediction
\citep{cortes2016structured}, neural network generalization
\citep{neyshabur2018towards}, and non-convex optimization
\citep{foster2018uniform,davis2018uniform}. In some applications, however, $\ls_2$ contraction may be loose if the
functions $\phi$ satisfy more favorable Lipschitz properties. For $K$-means clustering, the bound \pref{eq:maurer} leads to a rate of
$\bigO(K\sqrt{n})$ \citep{maurer2016vector}, while the optimal rate is $\bigOt(\sqrt{Kn})$
\citep{fefferman2016testing}. 

In this short note we prove a tighter vector contraction inequality
for $\ls_{\infty}$-Lipschitz functions. When $\phi_1,\ldots,\phi_n$ are $L$-Lipschitz with respect to
the $\ls_{\infty}$ norm,
(i.e. $\nrm*{\phi_t(x)-\phi_t(y)}_{\infty}\leq{}L\cdot\nrm*{x-y}_{\infty}$),
we show that
\begin{equation}
\label{eq:main_short}
\En_{\eps}\sup_{f\in\cF}\sum_{t=1}^{n}\eps_t\phi_t\prn*{f(x_t)}
\leq{} \bigOt\prn*{L\sqrt{K}}\cdot
\max_{i}\sup_{\xl[n]\in\cX}\En_{\eps}\sup_{f\in\cF}\sum_{t=1}^{n}\eps_{t}f_i(x_t).
\end{equation}
The result recovers the correct $\tilde{O}(\sqrt{Kn})$ rate for
$K$-means clustering, and generalizes a bound for $K$-fold maxima of
hyperplanes due to \cite{kontorovich2018rademacher}. Up to $\log{}n$ factors, the result cannot
be improved.


\section{The $\ls_{\infty}$ contraction inequality}
\label{sec:contraction}
To state the main result we require some additional notation. For a real-valued function
class $\cF$, let $\RadW(\cF)=\max_{\xl[n]\in\cX}\Rad(\cF\midsem\xl[n])$ denote the worst-case Rademacher complexity. When $\cF$ takes values
in $\bbR^{K}$, let $\cFi$ denote its restriction to output
coordinate $i$. 
\newcommand{\rangebound}{\beta}
\begin{theorem}[$\ls_{\infty}$ contraction inequality]
\label{thm:main}
Let $\cF\subseteq
\crl*{f:\cX\to\bbR^{K}}$, and let $\phi_1,\ldots,\phi_n$ each be $L$-Lipschitz with respect to the
$\ls_{\infty}$ norm. For any $\delta>0$, there exists a constant $C>0$
such that if
$\abs*{\phi_t(f(x))}\vee\nrm*{f(x)}_{\infty}\leq{}\rangebound$, then
  \begin{equation}
    \label{eq:main_theorem}
\Rad(\phi\circ\cF\midsem\xl[n])
\leq{} C\cdot{}L\sqrt{K}\cdot\max_{i}\RadW(\cFi)\cdot\log^{\frac{3}{2}+\delta}\prn*{\tfrac{\beta{}n}{\max_{i}\RadW(\cFi)}}.
  \end{equation}
\end{theorem}
\begin{proof}
We first gather some basic definitions. For a real-valued function
class $\cG\subseteq\crl*{g:\cX\to\bbR}$, the empirical $L_2$ covering
number $\cN_2(\cG,\veps,\xl[n])$ is 
the size of the smallest set of sequences $V\subseteq\bbR^{n}$ for
which
\[
\forall{}g\in\cG\;\;\exists{}v\in{}V\;\;\text{such that}\;\;\prn*{\frac{1}{n}\sum_{t=1}^{n}\prn*{g(x_t)-v_t}^{2}}^{1/2}\leq{}\veps.
\]
The empirical $L_{\infty}$ covering
number $\cN_{\infty}(\cG,\veps,\xl[n])$ is the size of the smallest set of sequences $V\subseteq\bbR^{n}$ for
which
\[
\forall{}g\in\cG\;\;\exists{}v\in{}V\;\;\text{such that}\;\;\max_{1\leq{}t\leq{}n}\abs*{g(x_t)-v_t}\leq{}\veps.
\]
Recall that $\cG$ is said to shatter $x_1,\ldots,x_n$ at
scale $\gamma$ if there exists a sequence $v_1,\ldots,v_n$ such that
\[
\forall{}\eps\in\pmo^{n}\;\;\exists{}g\in\cG\;\;\text{such
  that}\;\;\eps_t\cdot{}(g(x_t)-v_t) \geq{}\frac{\gamma}{2}\;\;\forall{}t.
\]
The fat-shattering dimension
\citep{alon1997scale,bartlett1998prediction} is then defined via
\[
\fat_{\gamma}(\cG) =
\max\crl*{n\mid{}\text{$\exists{}x_1,\ldots,x_n$ such that $\cG$
    $\gamma$-shatters $x_1,\ldots,x_n$}}.
\]
Our proof is based on technique introduced by \cite{srebro2010smoothness}: we first bound Rademacher complexity in terms of covering numbers,
then bound covering numbers in terms of fat-shattering dimension, and
finally bound fat-shattering dimension in terms of Rademacher
complexity. The key idea is to use $\ls_{\infty}$-Lipschitzness to
remove the functions $\phi_1,\ldots\phi_n$ at the covering number
stage, then come back to Rademacher complexity. 

We prove the theorem for the case $\beta=L=1$; the main theorem
statement follows by applying this result after rescaling via $\bar{\phi}(v)\ldef{}\frac{1}{\beta{}L}\phi(\beta{}v)$ and $\bar{\cF}\ldef{}\cF/\beta$.

For the first step, by the standard chaining result (e.g. Theorem 2 of \cite{srebro2010note}), we have
\begin{equation}
\Rad(\phi\circ\cF\midsem\xl[n])
\leq{} \inf_{\alpha>0}\crl*{
4\alpha{}n + 12\sqrt{n}\int_{\alpha}^{1}
\sqrt{\log\cN_{2}(\phi\circ\cF,\veps,\xl[n])}d\veps
}.\label{eq:dudley}
\end{equation}
Next, we use the $\ls_{\infty}$-Lipschitz property  to prove the following lemma.
\begin{lemma}For all $\veps>0$ and $\xl[n]$ it holds that
\label{lem:covering_bound}
  \begin{equation}
    \label{eq:covering bound}
    \log\cN_{2}(\phi\circ\cF,\veps,\xl[n])
\leq{} K\cdot{}\max_{i}\log\cN_{\infty}(\cFi,\veps,\xl[n]).
  \end{equation}
\end{lemma}
\begin{proof}
Let a sequence $v_1,\ldots,v_n\in\bbR^{K}$ and element $f\in\cF$ be fixed. Then we have
\[
\prn*{\frac{1}{n}\sum_{t=1}^{n}\prn*{\phi_t(f(x_t))-\phi_t(v_t)}^{2}}^{1/2}
\leq{} \max_{1\leq{}t\leq{}n} \abs*{\phi_t(f(x_t))-\phi_t(v_t)}
\leq{} \max_{1\leq{}t\leq{}n} \nrm*{f(x_t)-v_t}_{\infty}.
\]
Consequently, if sets $V_1,\ldots,V_K$ each witness the $L_{\infty}$
covering numbers for $\cF|_1,\ldots,\cF|_{K}$ at scale $\veps$, their
cartesian product witnesses the $L_2$ covering number for
$\phi\circ\cF$ at scale $\veps$. The result follows because the
cartesian product has size most $\max_{i}\abs*{V_i}^{K}$.
\end{proof}
\pref{lem:covering_bound} allows us to bound the entropy integral \pref{eq:dudley} in terms
of covering numbers for the coordinate restrictions of $\cF$. All we require now are the
following lemmas, which will allow us to bound the resulting entropy
integral by the Rademacher complexity of each class $\cFi$.
\begin{lemma}[\cite{rudelson2006combinatorics}, Theorem 4.4\footnote{The result is stated in \cite{rudelson2006combinatorics} with a restriction that $\veps<1/2$, but the range can be extended to get the result here by rescaling the function class.}]
\label{lem:rudelson}
For any $\delta\in(0,1)$ there exist constants $0<c<1$ and $C\geq{}0$ such that
for all $\veps\in(0,1)$,
\[
\log\cN_{\infty}(\cFi,\veps,\xl[n])\leq{}Cd_{i}\log\prn*{en/d_i\veps}\log^{\delta}(en/d_{i}),
\]
where $d_i=\fat_{c\veps}(\cFi)$.
\end{lemma}
\begin{lemma}[\cite{srebro2010smoothness}, Lemma A.2]
\label{lem:fat_rademacher}
For all $\veps\geq{}\frac{2}{n}\RadW(\cFi)$, it holds that 
\[
\fat_{\veps}(\cFi)\leq{}\frac{8}{n}\cdot{}\prn*{\frac{\RadW(\cFi)}{\veps}}^{2},\quad\text{and}\quad
\fat_{\veps}(\cFi)\leq{}n.\]
\end{lemma}
Define $r_i = \frac{8}{n}\prn*{\frac{\RadW(\cFi)}{c\veps}}^{2}$, with $c$ is as in \pref{lem:rudelson}. Applying the three lemmas in sequence, we have that for any
$\frac{2}{cn}\RadW(\cFi)\leq{}\veps<1$,
\vspace{-10pt}
\begin{align*}
  \log\cN_{2}(\phi\circ\cF,\veps,\xl[n])
&\overset{\text{(i)}}{\leq}  \max_{i}K\log\cN_{\infty}(\cFi,\veps, \xl[n]) \\
&\overset{\text{(ii)}}{\leq}\max_{i}CKd_{i}\log\prn*{en/d_i\veps}\log^{\delta}(en/d_{i})\\
&\overset{\text{(iii)}}{\leq}\max_{i}CKr_{i}\log\prn*{e^{2+\delta}n/r_i\veps}\log^{\delta}(e^{2+\delta}n/r_{i})\\
&\overset{\text{(iv)}}{\leq}\max_{i}\frac{C_1K}{n}\prn*{\frac{\RadW(\cFi)}{\veps}}^{2}\log^{1+\delta}\prn*{n/\RadW(\cFi)},
\end{align*}
where $C_1>0$ is a numerical constant. Here inequalities (i) and (ii) use \pref{lem:covering_bound} and \pref{lem:rudelson} respectively. Inequality (iii) uses that $d_i\leq{}n$, $r_i\leq{}2n$, and that for any
$a,b>0$, the function $x\mapsto{}x\log(a/x)\log^{\delta}(b/x)$ is
non-decreasing as long
as $a\geq{}b\geq{}e^{1+\delta}x$. Lastly, inequality (iv) is a direct
calculation after expanding $r_i$ and using that $\veps<1/e$ and $\RadW(\cFi)\leq{}n$.

To apply this bound in equation \pref{eq:dudley}, we adopt the
shorthand $\bar{r}=\max_{i}\RadW(\cFi)$ and choose $\alpha=\frac{2}{cn}\bar{r}$, which gives
\begin{align*}
\Rad(\phi\circ\cF\midsem\xl[n])
&\leq{} 
\frac{8}{c}\bar{r} + C_2\sqrt{K}\bar{r}\log^{(1+\delta)/2}(n/\bar{r})\int_{\frac{2}{cn}\bar{r}}^{1}
\veps^{-1}d\veps\\
&\leq{} 
  C_3\sqrt{K}\bar{r}\log^{\frac{3}{2}+\frac{\delta}{2}}(n/\bar{r}).
\end{align*}

\end{proof}
\subsection{Optimality}
The right-hand side in \pref{thm:main} scales with the \emph{worst-case}
Rademacher complexity $\RadW(\cFi)$ rather than the more favorable
empirical Rademacher complexity $\Rad(\cFi\midsem\xl[n])$. At first
glance this looks peculiar, as both the classical Lipschitz contraction
inequality and Maurer's $\ls_2$ vector variant preserve the empirical Rademacher complexity. Perhaps surprisingly,
it turns out that \pref{thm:main} cannot be improved to depend on the
empirical Rademacher complexity without incurring an additional
$\sqrt{K}$ factor.
\begin{proposition}
\label{prop:lb}
There exists a set $\cX$, function class $\cF\subseteq\crl*{f:\cX\to\bbR^{K}}$, $1$-Lipschitz (w.r.t. $\ls_{\infty}$)
function $\phi$, and data sequence $\xl[n]$ for which
\[
\Rad(\phi\circ\cF\midsem\xl[n]) \geq{} 8^{-1/2}K\cdot{}\max_{i}\Rad(\cFi\midsem\xl[n]).
\]
\end{proposition}
It follows from \pref{eq:maurer} that $\Rad(\phi\circ\cF\midsem\xl[n]) \leq{} \sqrt{2}K\cdot{}\max_{i}\Rad(\cFi\midsem\xl[n])$, so we have a tradeoff: We can either upper bound by empirical Rademacher at a cost of $O(K)$, or upper bound by worst-case Rademacher at a cost of $\tilde{O}(\sqrt{K})$ using \pref{thm:main}.
\begin{proof}
Let $K$ be fixed. We define $\cX=\crl*{e_i}_{i\leq{}K}$, where $e_i$ denotes the
$i$th standard basis vector in $\bbR^{K}$. We choose the hypothesis class
\[
\cF=\crl*{x\mapsto{}\prn*{\sigma_1\tri*{e_1,x},\ldots,\sigma_K\tri*{e_K,x}}\mid{}\sigma_i\in\pmo},
\]
and select $\phi(v)=\max_{k}\crl*{v_k}$. 

Let $n$ be divisible by $K$. We select $x_{1},\ldots,x_{n/K}=e_1$,
$x_{n/K},\ldots,x_{2n/K}=e_2$, and so on, and let $i_t$ be such that
$x_t=e_{i_t}$. We can write the empirical Rademacher complexity as
\begin{align*}
  \Rad(\phi\circ\cF\midsem\xl[n]) 
                   &=
                     \En_{\eps}\max_{\sigma_1,\ldots,\sigma_K\in\pmo}\sum_{t=1}^{n}\eps_t\max\crl*{\sigma_1\tri*{e_1,x_t},\ldots,\sigma_K\tri*{e_K,x_t}}\\
                   &=
                     \En_{\eps}\max_{\sigma_1,\ldots,\sigma_K\in\pmo}\sum_{t=1}^{n}\eps_t\max\crl*{\sigma_{i_t},0},\\
                   &= K\cdot{}\En_{\eps}\max_{\sigma\in\pmo}\sum_{t=1}^{n/K}\eps_t\max\crl*{\sigma,0}.
\end{align*}
Using the Khintchine inequality \citep{haagerup1981best}, we have a lower bound
\[
\En_{\eps}\max_{\sigma\in\pmo}\sum_{t=1}^{n/K}\eps_t\max\crl*{\sigma,0}
= \frac{1}{2}\En_{\eps}\abs*{\sum_{t=1}^{n/K}\eps_t}\geq \sqrt{\frac{n}{8K}}.
\]
On the other hand, for each $i$, direct calculation shows that
\[
  \Rad(\cFi,\xl[n])
  =\En_{\eps}\max_{\sigma\in\pmo}\sum_{t=1}^{n}\eps_t\sigma\tri*{e_i,x_{t}}=\En_{\eps}\abs*{\sum_{t=1}^{n/K}\eps_t}\leq{}\sqrt{\frac{n}{K}}.
\]
This proves the result. Note that there is no contradiction with
\pref{thm:main}, since the worst case Rademacher complexity for this construction is large: $\RadW(\cF_i)\geq{}\sqrt{n/2}$.
\end{proof}

\subsection{Extension to $\ls_p$ norms}
The final result we prove is a generalization of \pref{thm:main} to $\ls_p$ norms beyond $\ls_{\infty}$.
\begin{theorem}[$\ls_p$ contraction inequality]
\label{thm:lp}
Let $\cF\subseteq
\crl*{f:\cX\to\bbR^{K}}$, and let $\phi_1,\ldots,\phi_n$ each be $L$-Lipschitz with respect to the
$\ls_{p}$ norm, where $p\in(0,\infty)$. For any $\delta>0$, there exists a constant $C>0$
such that if
$\abs*{\phi_t(f(x))}\vee\nrm*{f(x)}_{\infty}\leq{}\rangebound$, then
  \begin{equation}
    \label{eq:lp}
\Rad(\phi\circ\cF\midsem\xl[n])
\leq{} \bigOt(L)\cdot\prn*{\sum_{i=1}^{K}\RadW^{\frac{2p}{2+p}}(\cFi)}^{\frac{2+p}{2p}}.
  \end{equation}
\end{theorem}
\begin{proof}
We only sketch the result as the argument follows
\pref{thm:main}. As in the proof of \pref{thm:main}, assume
$\beta=L=1$ without loss of generality. Define covering scale parameters
\[
\veps_i =
\veps\cdot\prn*{\frac{\RadW^{\frac{2p}{2+p}}(\cFi)}{\sum_{j=1}^{K}\RadW^{\frac{2p}{2+p}}(\cF|_j)}}^{1/p}.
\]
 We
claim that  $\log\cN_{2}(\phi\circ\cF,\veps,\xl[n])
\leq{} \sum_{i=1}^{K}\log\cN_{\infty}(\cFi,\veps_i,\xl[n])$. Observe that
\[
\prn*{\frac{1}{n}\sum_{t=1}^{n}\prn*{\phi_t(f(x_t))-\phi_t(v_t)}^{2}}^{1/2}
\leq{} \max_{1\leq{}t\leq{}n} \abs*{\phi_t(f(x_t))-\phi_t(v_t)}
\leq{} \max_{1\leq{}t\leq{}n} \nrm*{f(x_t)-v_t}_{p}.
\]
Since $\prn*{\sum_{i=1}^{K}\veps_{i}^{p}}^{1/p}=\veps$, it follows
that if $V_1,\ldots,V_K$ are $L_{\infty}$ covers
for $\cF|_1,\ldots,\cF|_K$ at scales $\veps_1,\ldots,\veps_K$, their
cartesian product is a $\veps$-cover for $\phi\circ\cF$.

From here, we apply the same sequence of inequalities as in
\pref{thm:main}. After applying \pref{lem:rudelson} and
\pref{lem:fat_rademacher} in conjunction with the
covering number bound above, we have
\[
\sqrt{\log\cN_{2}(\phi\circ\cF,\veps,\xl[n])}
\approxleq\sqrt{\sum_{i=1}^{K}\frac{\RadW^{2}(\cFi)}{\veps_i^{2}}}
=
\veps^{-1}\cdot
\prn*{\sum_{i=1}^{K}\RadW^{\frac{2p}{2+p}}(\cFi)}^{\frac{2+p}{2p}}.
\]
Applying this inequality within the Dudley integral leads to the result.
\end{proof}


\paragraph*{Acknowledgements}
We thank Aryeh Kontorovich for posing a question that led us to this result, and for encouraging us to make the note available online. We are grateful to Ayush Sekhari and Karthik Sridharan for several helpful discussions.

\bibliography{refs}

\end{document}